\documentclass[12pt,a4paper,final]{iopart}

\usepackage{iopams}
\usepackage[dvipdfmx]{graphicx}
\usepackage[breaklinks=true,colorlinks=true,linkcolor=blue,urlcolor=blue,citecolor=blue]{hyperref}

\usepackage{setstack}
\usepackage{amsthm,bm,amsfonts}
\newcommand{\E}{\mathbb{E}}

\newcommand{\bx}{{\bf x}}
\newcommand{\bz}{{\bf z}}
\newtheorem{theorem}{Theorem}

\usepackage{lipsum}
\usepackage{subfigure}

\makeatletter
\renewcommand{\@makefnmark}{$^\ddagger$}
\makeatother

\usepackage[algoruled]{algorithm2e}
\usepackage{listings}
\usepackage{fancyvrb}
\fvset{fontsize=\normalsize}
\usepackage{booktabs}

\begin{document}

\title{Tightening Bounds for Variational Inference by Revisiting Perturbation Theory}

\author{Robert Bamler\footnote{Equal Contribution}}
\address{University of California, Irvine}
\ead{rbamler@uci.edu}

\author{Cheng Zhang$^\ddagger$}
\address{Microsoft Research, Cambridge}
\ead{Cheng.Zhang@microsoft.com}

\author{Manfred Opper }
\address{Technical University of Berlin}
\ead{manfred.opper@tu-berlin.de}

\author{Stephan Mandt$^\ddagger$}
\address{University of California, Irvine}
\ead{mandt@uci.edu}

\begin{abstract}
Variational inference has become one of the most widely used methods in latent variable modeling. In its basic form, variational inference employs a fully factorized variational distribution and minimizes its KL divergence to the posterior. As the minimization can only be carried out approximately, this approximation induces a bias. In this paper, we revisit perturbation theory as a powerful way of improving the variational approximation. Perturbation theory relies on a form of Taylor expansion of the log marginal likelihood, vaguely in terms of the log ratio of the true posterior and its variational approximation. While first order terms give the classical variational bound, higher-order terms yield corrections that tighten it.
However, traditional perturbation theory does not provide a lower bound, making it inapt for stochastic optimization. In this paper, we present a similar yet alternative way of deriving corrections to the ELBO that resemble perturbation theory, but that result in a valid bound.
We show in experiments on Gaussian Processes and Variational Autoencoders that the new bounds are more mass covering, and that the resulting posterior covariances are closer to the true posterior and lead to higher likelihoods on held-out data.
\end{abstract}

\vspace{2pc}
\noindent{\it Keywords}: Black Box Variational Inference, Perturbation Theory

\section{Introduction}

Bayesian inference is the task of reasoning about random variables that can only be measured indirectly.
Given a probabilistic model and a data set of observations, Bayesian inference seeks the posterior probability distribution over the remaining, i.e., unobserved or `latent' variables.
The computational bottleneck is calculating the marginal likelihood, the probability of the data with latent variables being integrated out.
The marginal likelihood is also an important quantity for model selection.
It is the objective function of the expectation-maximization (EM) algorithm~\cite{dempster1977maximum}, which has seen a recent resurgence with the popularity of variational autoencoders (VAEs)~\cite{kingma2013auto}.

Variational inference (VI)~\cite{jordan1999introduction,hoffman2013stochastic,zhang2018advances} scales up approximate Bayesian inference to large data sets by framing inference as an optimization problem. VI derives and maximizes a lower bound to the log marginal likelihood, termed 'evidence', and uses this so-called 'evidence lower bound' (ELBO)  as a proxy for the true evidence.
In its original formulation, VI was limited to a restricted class of so-called conditionally conjugate models~\cite{hoffman2013stochastic}, for which a closed-form expression for the ELBO can be derived by solving known integrals. More recently, black box variational inference (BBVI) approaches have become more popular~\cite{ranganath2014black,rezende2014stochastic}, which lift this restriction by approximating the gradient of the ELBO by Monte Carlo samples. 
BBVI also enables the optimization of alternative bounds that do not possess closed-form integrals. 
With Monte Carlo gradient estimation, the focus shifts from tractability of these integrals to the variance of the Monte Carlo gradient estimators, asking for bounds with low-variance gradients.

In this article, we propose a family of new lower bounds that are tighter than the standard ELBO while having lower variance than previous alternatives and admitting unbiased Monte Carlo estimation of the gradients of the bound.
We derive these new bounds by introducing ideas from so-called variational perturbation theory into BBVI.

Variational perturbation theory provides an alternative to VI for approximating the evidence~\cite{paquet2009perturbation,opper2013perturbative,opper2015perturbation,schwartz2008ideas}.
It is based on a Taylor expansion of the evidence around the variational distribution, i.e., a parameterized approximate posterior distribution.
The lowest order of the Taylor expansion recovers the standard ELBO, and higher order terms correct for the mismatch the between variational distribution and true posterior.

Variational perturbation theory based on so-called cumulant expansions~\cite{opper2015perturbation} requires a fair amount of manual derivations and puts strong constraints on the tractability of certain integrals. 
A cumulant expansion also generally does not result in a lower bound, making it impossible to minimize the approximation error with stochastic gradient descent.
In this article, we propose a new variant of perturbation theory that addresses these issues.
Our proposed perturbative expansion leads to a family of lower bounds on the marginal likelihood that can be optimized with the same black box techniques as the standard ELBO.
The proposed bounds ${\cal L}^{(K)}$ are enumerated by an odd integer~$K$, the order of the perturbative expansion, and are given by
\begin{equation} \label{eq:perturbative-bounds}
    \mathcal L^{(K)}(\lambda,V_0)
    = e^{-V_0} \sum_{k=0}^K \frac{1}{k!} \E_{\mathbf z\sim q_\lambda} \!\left[
            \left(V_0 + \log p(\bx,\bz) - \log q_\lambda(\bz)\right)^k
        \right]
    \leq p(\bx).
\end{equation}
Here, $p(\bx,\bz)$ is the joint distribution of the probabilistic model with observed variables~$\bx$ and latent variables~$\bz$, $q_\lambda(\bz)$ is the variational distribution with variational parameters~$\lambda$, and $p(\bx)$ is the marginal likelihood.
Further, the reference energy $V_0\in\mathbb R$ is an additional free parameter over which we optimize jointly with~$\lambda$.

The proposed bounds in Eq.~\ref{eq:perturbative-bounds} generalize the standard ELBO, which one obtains for~$K=1$.
Higher order terms make the bound successively tighter while guaranteeing that it still remains a lower bound for all odd orders~$K$.
In the limit~$K\to\infty$, Eq.~\ref{eq:perturbative-bounds} becomes an asymptotic series to the exact marginal likelihood~$p(\bx)$.

Our main contributions are as follows:
\begin{itemize}
\item
    We revisit variational perturbation theory~\cite{paquet2009perturbation,opper2013perturbative,opper2015perturbation} based on cumulant expansions, which does not result in a bound. By introducing a generalized variational inference framework, we propose a similar yet alternative construction to cumulant expansions which results in a valid lower bound of the evidence.
\item
    We furthermore show that the new bound admits unbiased Monte Carlo gradients with low variance, making it well suited to stochastic optimization. Unbiasedness is an important property and is necessary to guarantee convergence~\cite{robbins1951stochastic}. Biased gradients may destroy the bound and may lead to divergence problems~\cite{burda2015importance,dieng2016chi,domke2018importance}.
\item
    We evaluate the proposed method experimentally both for inference and for model selection.
    Already the lowest nonstandard order $K=3$ is less prone to underestimating posterior variances than standard VI, and it fits better models in a variational autoencoder (VAE) on small data sets.
    While this reduction in bias comes at the cost of larger gradient variance, we show experimentally and explain theoretically that the gradient variance is still smaller than in previously proposed alternatives~\cite{minka2005divergence,hernandez2016black,dieng2016chi} to the standard ELBO, resulting in faster convergence.
\end{itemize}
This paper is an extended version of a proceeding conference paper by the same authors~\cite{bamler2017perturbative}, which also gave the first reference to variational inference as a form of biased importance sampling, where the tightness of the bound is traded off against a low-variance stochastic gradient.

Our paper is structured as follows. In Section~\ref{sec:cumulant_new}, we revisit the cumulant expansion for variational inference. In Section~\ref{sec:method}, we derive a new family of bounds with similar properties, but which is amenable to stochastic optimization. We then discuss theoretical properties of the bounds. Experiments are presented in Section~\ref{sec:experiments}. Finally, we discuss related work in Section~\ref{sec:related} and open questions in Section~\ref{sec:conclusion}.

\section{Background: Perturbation Theory for Variational Inference}
\label{sec:cumulant_new}
We begin by reviewing variational perturbation theory as introduced in~\cite{opper2015perturbation}. We consider a probabilistic model with data~$\bx$, latent variables~$\bz$, and joint distribution $p(\bx,\bz)$.
The goal of Bayesian inference is to find the posterior distribution,
\begin{equation}\label{eq:posterior-and-px}
	p(\bz|\bx) = \frac{p(\bx,\bz)}{p(\bx)}
    \qquad\text{where}\qquad
	p(\bx) = \int p(\bx,\bz) \,d\bz.
\end{equation}

Exact posterior inference is impossible in all but the most simple models because of the intractable integral defining the marginal likelihood~$p(\bx)$ in Eq.~\ref{eq:posterior-and-px}.
Variational perturbation theory approximates the log marginal likelihood, or evidence,~$\log p(\bx)$, via a Taylor expansion.
We introduce a so-called variational distribution~$q_\lambda(\bz)$, parameterized by variational parameters~$\lambda$, and we write the evidence as follows:
\begin{equation} \label{eq:evidence-expv}
    \log p(\bx)
    = \log \left(\E_{\bz \sim q_\lambda} \left[
            \frac{p(\bx,\bz)}{q_\lambda(\bz)}
        \right] \right)
    = \left. \log \left( \E_{\bz \sim q_\lambda} \left[
            e^{-\beta V(\bx,\bz)}
        \right] \right) \right|_{\beta=1}.
\end{equation}
Above, the `interaction energy'~$V$ is defined as
\begin{equation} \label{eq:def-v}
    V(\bx,\bz) \equiv \log q_\lambda(\bz) - \log p(\bx,\bz),
\end{equation}
and the notation~$(\cdots)|_{\beta=1}$ on the right-hand side of Eq.~\ref{eq:evidence-expv} denotes that we introduced an auxiliary real parameter~$\beta$, which we set to one at the end of the calculation.

The reason for this notation is as follows. We approximate the right-hand side of Eq.~\ref{eq:evidence-expv} by a finite order Taylor expansion in~$\beta$ before setting $\beta=1$. While $\beta$ is not a small parameter, it is a placeholder that keeps track of appearances of $V(\bx,\bz)$ which we consider to be small for all~$\bz$ in the support of~$q_\lambda$. More precisely, it is enough to demand that $V(\bx,\bz)$ is almost constant in $\bz$ with small deviations, as is the case when $q_\lambda(\bz)$ is close to the posterior~$p(\bz|\bx)$. (This will become clear in Eq.~\ref{eq:cumulant-expansion} below.) Hence, we can think of this approximation informally as an expansion in terms of the difference between the log variational distribution and the log posterior.

The first order term of the Taylor expansion is the evidence lower bound, or ELBO,
\begin{equation} \label{eq:elbo}
    \mathcal L(\lambda)
    \equiv \E_{\bz \sim q_\lambda}[-V(\bx,\bz)]
    = \E_{\bz\sim q_\lambda}\left[
            \log p(\bx,\bz) - \log q_\lambda(\bz)
        \right].
\end{equation}
As is well known in variational inference (VI)~\cite{jordan1999introduction,blei2017variational,zhang2018advances}, the ELBO is a lower bound on the evidence, i.e., $\mathcal L(\lambda) \leq \log p(\bx)$ for all~$\lambda$.
VI maximizes the ELBO over~$\lambda$, and thus tries to find the best approximation of the evidence that is possible within a first order Taylor expansion in~$V$.

Higher order terms of the Taylor expansion lead to corrections to the standard ELBO.
Suppressing the dependence of~$V$ on~$\bx$ and~$\bz$ to simplify the notation, we obtain
\begin{eqnarray} \label{eq:cumulant-expansion}
    \log p(\bx)
    \approx &\, \E_{q_\lambda} \!\left[ -V \right]
        + \frac12 \E_{q_\lambda} \!\left[ \left(V  - \E_{q_\lambda} \!\left[ V  \right]\right)^2 \right]
        - \frac{1}{3!} \E_{q_\lambda} \!\left[ \left(V  - \E_{q_\lambda} \!\left[ V  \right]\right)^3 \right] \nonumber\\
    &
        + \frac{1}{4!} \E_{q_\lambda} \!\left[ \left(V  - \E_{q_\lambda} \!\left[ V  \right]\right)^4 \right]
        - \frac{1}{2} \left(\frac12 \E_{q_\lambda} \!\left[ \left(V  - \E_{q_\lambda} \!\left[ V  \right]\right)^2 \right] \right)^{\!\!2\,} \nonumber\\
    &
        + \ldots
\end{eqnarray}

Eq.~\ref{eq:cumulant-expansion} is called the \emph{cumulant expansion} of the evidence~\cite{opper2015perturbation}.
Due to the higher order correction terms, it can potentially approximate the true evidence better than the ELBO.

Variational perturbation theory has not found its way into mainstream machine learning, which arguably has to do with two major problems of the cumulant expansion and related approaches. 
First, in contrast to the ELBO, a higher order cumulant expansion does not result in a lower bound on the evidence.
This prevents the usage of gradient-based optimization methods (as opposed to, e.g., coordinate updates), which is currently the mainstream approach to VI.
A second drawback that will be discussed below is that the cumulant expansion cannot be estimated efficiently using Monte Carlo sampling, as we discuss in Section~\ref{sec:comparison} below.
In the following sections, we will present a similar construction of an improved variational objective that does not suffer from these shortcomings, and that is compatible with black-box optimization approaches.
To this end, we have to start from a generalized formulation of VI.

\section{Perturbative Black Box Variational Inference}
\label{sec:method}

This section presents our main results.
We derive a family of new objective functions for Black Box Variational Inference (BBVI) which are amenable to stochastic optimization.
We begin by deriving a generalized ELBO for VI based on concave functions.
We then show that, in a special case, we obtain a strict bound with resemblance to variational perturbation theory.

\subsection{Variational Inference with Generalized Lower Bounds}
\label{sec:generalized-bounds}

Variational inference (VI) approximates the evidence~$\log p(\bx)$ of the model in Eq.~\ref{eq:posterior-and-px} by maximizing the ELBO~$\mathcal L(\lambda)$ (Eq.~\ref{eq:elbo}) over  variational parameters~$\lambda$. A more common approach to derive the ELBO is through Jensen's inequality~\cite{jordan1999introduction,blei2017variational,zhang2018advances}. Here, we present a similar derivation for a broader family of lower bounds.

To this end, we consider an arbitrary concave function~$f$ over the positive reals, i.e., $f:\mathbb R_{>0} \to\mathbb R$ with second derivative $f''(\xi)\leq0\;\forall\xi\in\mathbb R_{>0}$.
As an example, $f$ could be the logarithm.
We now consider the marginal likelihood $p(\bx)=\E_{\bz\sim q_\lambda}[p(\bx,\bz)/q_\lambda(\bz)]$, and we derive a lower bound on $f(p(\bx))$ using Jensen's inequality:
\begin{equation} \label{eq:f-bound}
    f(p(\bx))
    \geq \E_{\bz\sim q_\lambda} \left[
            f\left(\frac{p(\bx,\bz)}{q_\lambda(\bz)}\right)
        \right]
    =: \mathcal L_f(\lambda).
\end{equation}

As stated above, for $f(\cdot) = \log(\cdot)$ (which is a concave function), the bound in Eq.~\ref{eq:f-bound} is the standard ELBO from Eq.~\ref{eq:elbo}.
In this case, we refer to the corresponding VI scheme as KLVI, as maximizing the standard ELBO is equivalent to minimizing the Kullback-Leibler (KL) divergence from~$q_\lambda(\bz)$ to the true posterior~$p(\bz|\bx)$.

The above class of generalized bounds~$\mathcal L_f(\lambda)$ is compatible with current mainstream optimization schemes for VI, oftentimes summarized as Black Box Variational Inference (BBVI)~\cite{rezende2014stochastic,ranganath2014black}.
For the reader's convenience, we briefly review these techniques.
BBVI optimizes the ELBO via stochastic gradient descent (SGD) based on noisy estimates of its gradient.
Crudely speaking, BBVI obtains its gradient estimates in three steps: (1) by drawing Monte Carlo (MC) samples $\bz\sim q_\lambda$, (2) by averaging the bound  over these samples, and (3) by computing the gradient on the resulting average.
A slight complication arises as both the expression inside the expectation in Eq.~\ref{eq:f-bound} as well as the distribution of MC samples itself depend on~$\lambda$. The two main solutions to this problem are as follows:
\begin{itemize}
\item
    Score function gradients~\cite{ranganath2014black}, or the REINFORCE method, use the chain rule to relate the change in MC samples to a change in the log variational distribution,
    \begin{equation} \label{eq:reinforce-grad}
        \nabla_{\!\lambda} \mathcal L_f(\lambda)
        = \E_{\bz \sim q_\lambda}\left[
                (\nabla_{\!\lambda} \log q_\lambda(\bz))
                \; f\!\left(\frac{p(\bx,\bz)}{q_\lambda(\bz)}\right)
            + \nabla_{\!\lambda}
                f\!\left(\frac{p(\bx,\bz)}{q_\lambda(\bz)}\right)
            \right].
    \end{equation}
\item
    Reparameterization gradients~\cite{kingma2013auto} can be used if the samples~$\bz\sim q_\lambda$ can be expressed as some deterministic differentiable transformation $\bz = g(\boldsymbol\epsilon;\lambda)$ of auxiliary random variables~$\boldsymbol\epsilon$ drawn from some $\lambda$-independent distribution~$\tilde q(\boldsymbol\epsilon)$.
    Under these conditions, the gradient of the bound can be written as
    \begin{equation} \label{eq:reparam-grad}
        \nabla_{\!\lambda} \mathcal L_f(\lambda)
        = \E_{\boldsymbol\epsilon \sim \tilde q} \left[
                \nabla_{\!\lambda}
                f\!\left(\frac{p(\bx,g(\boldsymbol\epsilon; \lambda))}{q_\lambda(g(\boldsymbol\epsilon; \lambda))}\right)
            \right].
    \end{equation}
\end{itemize}
One obtains an unbiased estimate of~$\nabla_{\!\lambda} \mathcal L_f(\lambda)$ by approximating the expectation on the right-hand side of Eq.~\ref{eq:reinforce-grad} or~\ref{eq:reparam-grad} via MC samples~$\bz\sim q_\lambda$ or~$\boldsymbol\epsilon \sim \tilde q$, respectively.
Empirically, reparameterization gradients often have smaller variance than score function gradients.

Based on these new bounds~$\mathcal L_f(\lambda)$, we make the following observations.

\paragraph{Variational Inference and Importance Sampling.}
The generalized family of bounds, Eq.~\ref{eq:f-bound}, shows that VI and importance sampling are closely connected~\cite{bamler2017perturbative,domke2018importance}.
For $f(\cdot)$ being the identity, the variational bound becomes independent of $\lambda$ and instead becomes an unbiased importance sampling estimator of~$p(\bx)$, with the proposal distribution $q_\lambda(\bz)$ and the importance ratio $p(\bx,\bz)/q_\lambda(\bz) = e^{-V(\bx,\bz)}$.
For other choices of~$f$, the bound is no longer an unbiased estimator of~$p(\bx)$.
This way, we can identify BBVI as a form of biased importance sampling.

\paragraph{Bias Variance Trade-off.}
Importance sampling suffers from high variance in high dimensions because the log importance ratio~$V(\bx,\bz)$ (Eq.~\ref{eq:def-v}) scales approximately linearly with the size of the model.
(To see this, consider a setup where both $p$ and $q$ factorize over all~$N$ data points, in which case $V$ is proportional to $N$.)

The choice of function~$f$ in Eq.~\ref{eq:f-bound} trades off some of this variance for a bias.
KLVI uses $f(\cdot)=\log(\cdot)$, which leads to low variance gradient estimates in Eqs.~\ref{eq:reinforce-grad}-\ref{eq:reparam-grad} that depend only linearly on~$V$, rather than exponentially.
This makes the KLVI bound easy to optimize, at the cost of some bias.
An alternative to KLVI that has been explored in the literature is $\alpha$-VI~\cite{minka2005divergence, hernandez2016black,li2016renyi,dieng2016chi}, which corresponds to setting $f(\xi)=\xi^{1-\alpha}$ with some $\alpha\in(0,1)$, and thus $f(e^{-V})=e^{-(1-\alpha)V}$. This choice has an alternative bias-variance trade-off, as $V$ enters in the exponent, leading to more variance than KLVI.
Our empirical results in Figure~\ref{fig:variances}, discussed in Section~\ref{sec:GP_classification}, confirm this picture by showing that $\alpha$-VI  suffers from slow convergence due to large gradient variance.

According to our new findings, a good alternative bound should perform favorably on the bias-variance trade-off.
It should ideally depend on~$V$ only polynomially as opposed to exponentially.
Such a family of bounds will be introduced next.
We also show how they connect to variational perturbation theory (Section~\ref{sec:cumulant_new}).

\begin{figure}
    \centering
    \includegraphics[width=0.5\columnwidth]{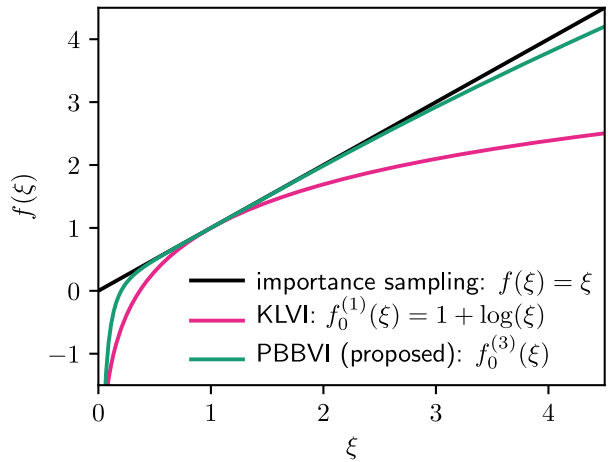}
    \caption{Different choices for the concave fucntion~$f(\xi)$ in Eq.~\ref{eq:f-bound}.
    For $f(\xi)=\xi$ (black), the bound is tight but independent of $\lambda$.
    KLVI corresponds to $f(\xi)=\log(\xi) + \text{const.}$ (pink).
    Our proposed PBBVI bound uses $f_{V_0}^{(K)}(\xi)$ (green, Eq.~\ref{eq:f-perturb}), which is tighter than KLVI for $\xi\approx e^{-V_0}$ (we set $K=3$ and $V_0=0$ for PBBVI here).}
    \label{fig:f_lower_bound}
\end{figure}
\subsection{Perturbative Lower Bounds}
\label{sec:perturbative-bounds}

We now propose a family of concave functions~$f$ for the bound in Eq.~\ref{eq:f-bound} that is inspired by ideas from perturbation theory (Section~\ref{sec:cumulant_new}).
The choice of the function~$f$ determines the tightness of the bound~$\mathcal L_f(\lambda)$.
A tight bound, i.e., a good approximation of the marginal likelihood, is important for the variational expectation-maximization (EM) algorithm, which uses the bound as a proxy for the marginal likelihood.
The bound~$\mathcal L_f(\lambda)$ becomes perfectly tight if we set~$f=\text{id}$ to the identity function, $\text{id}(\xi)=\xi$ (black line in Figure~\ref{fig:f_lower_bound}).
However, as we discussed, this is a singular limit in which the bound~$\mathcal L_{\text{id}}(\lambda) = p(\bx)$ does not depend on~$\lambda$, and therefore does not provide any gradient signal to learn an optimal variational distribution~$q_\lambda$.

Ideally, the bound should be as tight as possible if~$q_\lambda(\bz)$ is close to the posterior~$p(\bz|\bx)$.
In order to obtain a gradient signal, the bound should be lower if $q_\lambda(\bz)$ is very different from the posterior.
To achieve this, we recall that the argument~$\xi$ of the function~$f(\xi)$ in Eq.~\ref{eq:f-bound} is the fraction $p(\bx,\bz)/q_\lambda(\bz)$, which equals~$p(\bx)$ if $q_\lambda(\bz)$ is the true posterior.
Thus, the concave function~$f$ should be close to the identity function for arguments close to~$p(\bx)$.
Since the only things we know about~$p(\bx)$ are that it is positive and independent of~$\bz$, we parameterize it as~$e^{-V_0}$ with a real parameter~$V_0$, over which we optimize.

For any~$V_0\in\mathbb R$, and any odd integer $K\geq1$, we propose the following concave function:
\begin{eqnarray} \label{eq:f-perturb}
    f_{V_0}^{(K)}(\xi)
    &= e^{-V_0} \sum_{k=0}^K \frac{1}{k!} \left(V_0 + \log \xi\right)^k
    \;\overset{K\to\infty}{\longrightarrow}\;
    \text{id}(\xi)
    \qquad\text{(pointwise $\forall \xi$)}.
\end{eqnarray}

The pink and green lines in Figure~\ref{fig:f_lower_bound} plot $f_{V_0}^{(K)}(\xi)$ for $V_0=0$ and for $K=1$ and $K=3$, respectively.
The function~$f_{V_0}^{(K)}$ converges pointwise to the identity function (black line in Figure~\ref{fig:f_lower_bound}) for $K\to\infty$ because it is the $K$\textsuperscript{th} order Taylor expansion of~$\text{id}(\xi) = e^{-V_0}e^{V_0 + \log\xi}$ in $\log\xi$ around the reference point~$-V_0$.
For any finite~$K$, the Taylor expansion approximates the identity function (black line in Figure~\ref{fig:f_lower_bound}) most closely if the expansion parameter $|V_0 + \log\xi|$ is small.
We provide further intuition for Eq.~\ref{eq:f-perturb} in Section~\ref{sec:comparison} below.

An important result of this section is that the functions~$f_{V_0}^{(K)}$ define a lower bound on the marginal likelihood for any choice of $\lambda$ and $V_0$,
\begin{equation} \label{eq:chain-of-inequalities}
    p(\bx) \geq f(p(\bx)) \geq \mathcal L_{f_{V_0}^{(K)}}(\lambda)
    =: \mathcal L^{(K)}(\lambda,V_0)
    \qquad\text{(for odd $K$)}.
\end{equation}
Here, the last equality merely introduces a more convenient notation for the bound defined by Eqs.~\ref{eq:f-bound} and~\ref{eq:f-perturb}.
An explicit form for~$\mathcal L^{(K)}(\lambda,V_0)$ was given in Eq.~\ref{eq:perturbative-bounds} in the introduction.
Eq.~\ref{eq:chain-of-inequalities} follows from Eq.~\ref{eq:f-bound} and from the fact that~$f_{V_0}^{(K)}$ is concave and lies below the identity function, which we prove formally in Section~\ref{sec:proofs} below.
Since~$\mathcal L^{(K)}(\lambda,V_0)$ is a lower bound on~$p(\bx)$ for all~$V_0$, we can find the optimal reference point~$V_0$ that leads to the tightest bound by maximizing~$\mathcal L^{(K)}(\lambda,V_0)$ jointly over both~$\lambda$ and~$V_0$.
We also prove in Section~\ref{sec:proofs} that the bound on $p(\bx)$ is nontrivial, i.e., that $\mathcal L^{(K)}(\lambda,V_0)$ is positive at its maximum.
(A negative bound would be vacuous since~$p(\bx)$ is trivially larger than any negative number.)

\begin{algorithm}[t]
\KwIn{joint probability $p(\bx,\bz)$;
 order of perturbation $K$ (odd integer);
 learning rate schedule $\rho_t$;
 number of Monte Carlo samples $S$;$\qquad$
 number of training iterations $T$;
 variational family $q_\lambda(\bz)$ that admits reparameterization gradients as in Eq.~\ref{eq:reparam-grad}.}\vspace{1pt}
\KwOut{fitted variational parameters $\lambda^*$.}\vspace{2pt}
\nl Initialize $\lambda$ randomly and $V_0\gets0$\;
\nl \For{$t\gets1$ \KwTo $T$}{
\nl  Draw $S$ samples $\boldsymbol\epsilon_1,\ldots,\boldsymbol\epsilon_S \sim \tilde q$ from the noise distribution (see Eq.~\ref{eq:reparam-grad})\;\vspace{5pt}
  \textit{// Obtain gradient estimates of surrogate objective $\tilde{\mathcal L}^{(K)}(\lambda,V_0)$, see Eq.~\ref{eq:surrogate-objective}:}\\
\nl  $g_\lambda \hspace{3.9pt} \gets \nabla_{\!\lambda}\hspace{3.9pt} \!\Big[ \frac{1}{S} \sum_{s=1}^S \sum_{k=0}^{K} \frac{1}{k!}\big( V_0 + \log p(\bx,g(\boldsymbol\epsilon_s;\lambda)) - \log q_\lambda(g(\boldsymbol\epsilon_s;\lambda)) \big)^k\Big]$\label{ln:start-grad}\;
\nl  $g_{V_0} \gets \nabla_{\!V_0} \!\Big[\frac{1}{S} \sum_{s=1}^S \sum_{k=0}^{K} \frac{1}{k!} \big( V_0 + \log p(\bx,g(\boldsymbol\epsilon_s;\lambda)) - \log q_\lambda(g(\boldsymbol\epsilon_s;\lambda)) \big)^k$\Big]\;\vspace{5pt}
  \textit{// Perform update steps with rescaled gradients, see Eq.~\ref{eq:scaled-grads}:}\\
\nl  $\lambda \hspace{4.4pt}\gets \lambda + \rho_t g_\lambda$\label{ln:update-lambda}\;
\nl  $V_0 \gets V_0 + \rho_t \Big[ g_{V_0} - \frac{1}{S} \sum_{s=1}^S \sum_{k=0}^{K} \frac{1}{k!} \big( V_0 + \log p(\bx,g(\boldsymbol\epsilon_s;\lambda)) - \log q_\lambda(g(\boldsymbol\epsilon_s;\lambda)) \big)^k\Big]$\label{ln:update-v0}\label{ln:end-grad}\;
 }
 \caption{Perturbative Black Box Variational Inference (PBBVI)}
 \label{alg:pbbvi}
\end{algorithm}

To optimize the bound, we use SGD with reparameterization gradients, see Section~\ref{sec:generalized-bounds}.
Score function gradients can be used in a similar way.
Algorithm~\ref{alg:pbbvi} explains the optimization in detail.
We name the method `Perturbative Black Box Variational Inference' (PBBVI).
The algorithm implicitly scales all gradients by a factor of~$e^{V_0}$ to avoid a potential exponential explosion or decay of gradients due to the factor~$e^{-V_0}$ in the bound, Eq.~\ref{eq:perturbative-bounds}.
We calculate these rescaled gradients in a numerically stable way by considering the surrogate objective function
\begin{eqnarray} \label{eq:surrogate-objective}
    \tilde{\mathcal L}^{(K)}(\lambda,V_0)
    &\equiv e^{V_0}\mathcal L^{(K)}(\lambda,V_0) \nonumber\\
    &= \sum_{k=0}^K \frac{1}{k!} \E_{\bz\sim q_\lambda}\left[\big(
        V_0 + \log p(\bx,\bz) - \log q_\lambda(\bz)
        \big)^k \right].
\end{eqnarray}
The rescaled gradients are thus
\begin{eqnarray} \label{eq:scaled-grads}
    e^{V_0} \nabla_{\!\lambda} \mathcal L^{(K)}(\lambda,V_0)
    &= \nabla_{\!\lambda} \tilde{\mathcal L}^{(K)}(\lambda,V_0); \nonumber\\
    e^{V_0} \nabla_{\!V_0} \mathcal L^{(K)}(\lambda,V_0)
    &= \nabla_{\!V_0} \tilde{\mathcal L}^{(K)}(\lambda,V_0) - \tilde{\mathcal L}^{(K)}(\lambda,V_0).
\end{eqnarray}
Eqs.~\ref{eq:surrogate-objective}-\ref{eq:scaled-grads} allow us to estimate the rescaled gradients (see lines~\ref{ln:start-grad}-\ref{ln:end-grad} in Algorithm~\ref{alg:pbbvi}) without having to evaluate any potentially overflowing or underflowing expressions~$e^{\pm V_0}$.

\subsection{Connection to Variational Perturbation Theory}
\label{sec:comparison}

We now provide more intuition for the choice of the functions~$f_{V_0}^{(K)}$ in Eq.~\ref{eq:f-perturb} by comparing the resulting bounds, Eq.~\ref{eq:perturbative-bounds}, to the cumulant expansion, Eq.~\ref{eq:cumulant-expansion}.
We will show that the bounds enjoy similar benefits as the cumulant expansion while being valid lower bounds and providing unbiased Monte Carlo gradients.

Both the cumulant expansion and the proposed bounds~$\mathcal L^{(K)}(\lambda,V_0)$ are Taylor expansions in $V(\bx,\bz) = \log q_\lambda(\bz) - \log p(\bx,\bz)$.
The cumulant expansion in Eq.~\ref{eq:cumulant-expansion} is a Taylor expansion of the evidence~$\log p(\bx)$.
By contrast, the bound~$\mathcal L^{(K)}(\lambda,V_0)$ in Eq.~\ref{eq:perturbative-bounds} is a Taylor expansion of the marginal likelihood~$p(\bx)$, and the expansion starts from a reference point~$V_0$ over which we optimize.

Despite this difference, the two expansions are remarkably similar up to order~\hbox{$K=3$}.
Up to this order, the cumulant expansion, Eq.~\ref{eq:cumulant-expansion}, coincides with a Taylor expansion of~$p(\bx)$ in $(V-\E_{q_\lambda}[V])$.
It is thus a special case of the proposed lower bound~$\mathcal L^{(3)}(\lambda,V_0)$ if we set $V_0 = \E_{q_\lambda}[V]$. For $K>3$, the cumulant expansion contains additional terms that do not appear in the bound~$\mathcal L^{(K)}(\lambda,V_0)$, such as the last spelled out term on the right-hand side of Eq.~\ref{eq:cumulant-expansion}.
These terms contain products of expectations under~$q_\lambda$, which are difficult to estimate with Monte Carlo techniques without introducing an additional bias.
The bounds~$\mathcal L^{(K)}(\lambda,V_0)$, by contrast, depend only linearly on expectations under~$q_\lambda$, which means that they can be estimated without bias with a single sample from~$q_\lambda$.

In addition, the main advantage of $\mathcal L^{(K)}(\lambda,V_0)$ over the cumulant expansion is that $\mathcal L^{(K)}(\lambda,V_0)$ is a lower bound on the marginal likelihood for all~$\lambda$ and all~$V_0$.
This allows us to make the bound as tight as possible by maximizing it over~$\lambda$ and~$V_0$ using stochastic gradient estimates.
By contrast, variational perturbation theory with the cumulant expansion requires either optimizing towards saddle points, or optimizing a separate objective function, such as the ELBO, to find a good variational distribution~\cite{opper2015perturbation}.

The lack of a lower bound in the cumulant expansion is a particular limitation when performing model selection with the variational expectation-maximization (EM) algorithm.
Variational EM fits a model by maximizing an approximation of the marginal likelihood or the model evidence over model parameters.
Such an optimization can diverge if the approximation is not a lower bound.
This makes perturbative BBVI an interesting alternative for improved evidence approximations~\cite{burda2015importance,dieng2016chi,li2016renyi,domke2018importance,rainforth2018tighter}.

\subsection{Proofs of the Theorems}
\label{sec:proofs}

We conclude the presentation of the proposed bounds~$\mathcal L^{(K)}(\lambda,V_0)$ by providing proofs of the claims made in Section~\ref{sec:perturbative-bounds}.

\begin{theorem} \label{theorem:concave}
For all $V_0\in\mathbb R$ and all odd integers $K\geq1$, the function $f_{V_0}^{(K)}$ defined in Eq.~\ref{eq:f-perturb} satisfies the following properties:
\begin{enumerate}
\item[(i)]
    $f_{V_0}^{(K)}$ is concave; and
\item[(ii)]
    $f_{V_0}^{(K)}(\xi) \leq \xi\quad \forall \xi\in\mathbb R_{>0}$.
\end{enumerate}
As a corollary, Eq.~\ref{eq:chain-of-inequalities} holds.
\end{theorem}

\begin{proof}\leavevmode
\begin{enumerate}
\item[(i)]
    The first derivative of~$f_{V_0}^{(K)}$ is
    \begin{equation} \label{eq:first-derivative}
        {f_{V_0}^{(K)}}'(\xi)
        = e^{-V_0} \sum_{k=0}^{K-1} \frac{1}{k!}
            \frac{(V_0+\log \xi)^k}{\xi}.
    \end{equation}
    For the second derivative, the two contributions from the denominator and the enumerator cancel for all but the highest order term, and we obtain
    \begin{equation}
        {f_{V_0}^{(K)}}''(\xi)
        = -\frac{e^{-V_0}}{(K-1)!} \frac{(V_0+\log \xi)^{K-1}}{\xi^2} \leq 0
    \end{equation}
    which is nonnegative everywhere for odd $K$.
    Therefore,~$f_{V_0}^{(K)}$ is concave.
\item[(ii)]
    Consider the function $g(\xi)=f_{V_0}^{(K)}(\xi) - \xi$, which is also concave since it has the same second derivative as~$f_{V_0}^{(K)}$.
    Further,~$g$ has a stationary point at $\xi = e^{-V_0}$ since ${f_{V_0}^{(K)}}'(e^{-V_0})=1$, which can be verified by inserting into Eq.~\ref{eq:first-derivative}.
    For a concave function, a stationary point is a global maximum.
    Therefore, we have
    \begin{equation}
        g(\xi) \leq g(e^{-V_0}) = 0
        \qquad \forall \xi \in \mathbb R_{>0}
    \end{equation}
    which is equivalent to the proposition $f_{V_0}^{(K)}(\xi) \leq \xi\; \forall \xi\in\mathbb R_{>0}$.
\end{enumerate}
Thus, the first inequality in Eq.~\ref{eq:chain-of-inequalities} holds by property~(ii), and the second inequality in Eq.~\ref{eq:chain-of-inequalities} follows from property~(i) and Jensen's inequality.
\end{proof}

\begin{theorem} \label{theorem:nontrivial}
The bound on the positive quantity~$p(\bx)$ is nontrivial, i.e., the maximum value of the bound, $\max_{\lambda,V_0} \mathcal L^{(K)}(\lambda,V_0)$, is positive for all odd integers~$K\geq 1$.
\end{theorem}

\begin{proof}
At the maximum position $(\lambda^*,V_0^*)$ of~$\mathcal L^{(K)}(\lambda,V_0)$, its gradient is zero.
Taking the gradient with respect to~$V_0$ in Eq.~\ref{eq:perturbative-bounds} and using the product rule for the prefactor~$e^{-V_0}$ and the remaining expression, we find that all terms except the contribution from $k=K$ cancel.
We thus obtain, at the maximum position~$(\lambda^*,V_0^*)$,
\begin{equation}
    \E_{\bz\sim q_{\lambda^*}}\left[\big(V_0^* + \log p(\bx,\bz) - \log q_{\lambda^*}(\bz)\big)^K\right]
    = 0.
\end{equation}
Thus, when we evaluate~$\mathcal L^{(K)}$ at~$(\lambda^*,V_0^*)$, the term with $k=K$ on the right-hand side of Eq.~\ref{eq:perturbative-bounds} evaluates to zero, and we are left with
\begin{equation} \label{eq:l-at-opt}
    \mathcal L^{(K)}(\lambda^*, V_0^*)
    = e^{-V_0^*} \E_{\bz\sim q_{\lambda^*}}\left[ h\big(V_0^* + \log p(\bx,\bz) - \log q_{\lambda^*}(\bx,\bz) \big) \right]
\end{equation}
with
\begin{equation} \label{eq:def-h}
    h( u) = \sum_{k=0}^{K-1} \frac{ u^k}{k!}
\end{equation}
where the sum runs only to $K-1$.

We now show that $h( u)$ is positive for all~$u$ and all odd~$K$.
If $K=1$, then $h( u)=1$ is trivially positive.
For $K\geq3$, $h( u)$ is a polynomial in $ u$ of even order $K-1$, whose highest order term has a positive coefficient $\frac{1}{(K-1)!}$.
Therefore, $h(u)$ goes to positive infinity for both $u\to\infty$ and $u\to-\infty$.
It thus has a global minimum at some value $\tilde u \in\mathbb R$.
At the global minimum, its derivative is zero, i.e.,
\begin{equation} \label{eq:derivative-h}
    0 = h'(\tilde  u) = \sum_{k=0}^{K-2} \frac{\tilde u^k}{k!}.
\end{equation}
Subtracting Eq.~\ref{eq:derivative-h} from Eq.~\ref{eq:def-h}, we find that the value of $h$ at its global minimum~$\tilde u$ is
\begin{equation} \label{eq:htildeu}
    h(\tilde u) = \frac{\tilde u^{K-1}}{(K-1)!} \geq 0
\end{equation}
which is nonnegative because $K-1$ is even.
Further, $h(\tilde u)$ is not zero since this would imply~$\tilde u = 0$ by Eq.~\ref{eq:htildeu}, but $h(0)=1$ by Eq.~\ref{eq:def-h}.
Therefore, $h(\tilde u)$ is strictly positive, and since $\tilde u$ is a global minimum of $h$, we have $h(u) > 0$ for all $u\in\mathbb R$.
Inserting into Eq.~\ref{eq:l-at-opt} concludes the proof that the lower bound at the optimum is positive.
\end{proof}

\section{Experiments}
\label{sec:experiments}
We evaluate PBBVI with different models.
We begin with a qualitative comparison on a simple one-dimensional toy problem (Section~\ref{sec:multimodal-1d}).
We then investigate the behavior of BBVI in a controlled setup of Gaussian processes on synthetic data (Section~\ref{sec:synthetic}).
Next, we evaluate PBBVI based on a classification task using Gaussian processes classifiers, where we use data from the UCI machine learning repository (Section~\ref{sec:GP_classification}).
This is a Bayesian non-conjugate setup where black box inference is required. Finally, we use an experiment with the variational autoencoder (VAE) to explore our approach on a deep generative model (Section~\ref{sec:VAE}). This experiment is carried out on  MNIST data.
We use the perturbative order $K=3$ for all experiments with PBBVI.
This corresponds to the lowest order beyond standard VI with the KL divergence (KLVI), since $K$ has to be an odd integer, and PBBVI with $K=1$ is equivalent to KLVI.
Across all the experiments, PBBVI demonstrates advantages based on different metrics.

\begin{figure}
    \centering
    \includegraphics[width=0.5\textwidth]{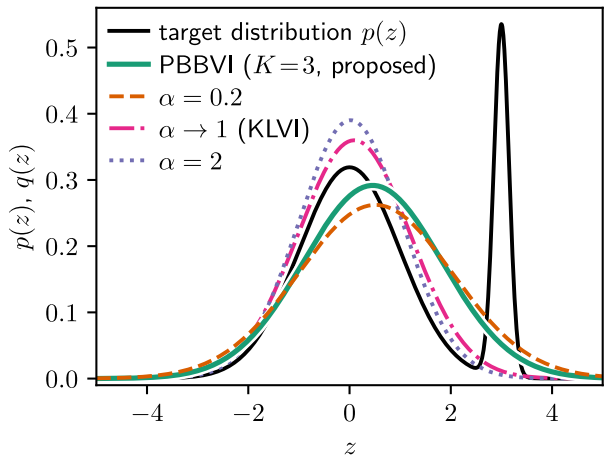}
    \caption{Behavior of different VI methods on fitting a univariate Gaussian to a bimodal target distribution (black). PBBVI (proposed, green) covers more of the mass of the entire distribution than the traditional KLVI (pink). $\alpha$-VI is mode seeking for large~$\alpha$ (blue) and mass covering for smaller~$\alpha$ (orange).}
    \label{fig:1DGaussian}
\end{figure}

\subsection{Mass Covering Effect}
\label{sec:multimodal-1d}

In Figure~\ref{fig:1DGaussian}, we fit a Gaussian distribution to a one-dimensional bimodal target distribution (black line), using different divergences.
Compared to BBVI with the standard KL divergence (KLVI, pink line), $\alpha$-divergences~\cite{minka2005divergence} are more mode-seeking (purple line) for large values of $\alpha$, and more mass-covering (orange line) for small $\alpha$~\cite{li2016renyi} (the limit $\alpha\to 1$ recovers KLVI~\cite{minka2005divergence}).
Our PBBVI bound ($K=3$, green line) achieves a similar mass-covering effect as in $\alpha$-divergences, but with associated low-variance reparameterization gradients.
This is seen in Figure~\ref{fig:variances}~(b), discussed in Section~\ref{sec:GP_classification} below, which compares the gradient variances of $\alpha$-VI and PBBVI as a function of latent dimensions.

\subsection{GP Regression on Synthetic Data}
\label{sec:synthetic}

\begin{figure}[t]
\centering
\subfigure[KLVI]{\includegraphics[width=0.42 \textwidth]{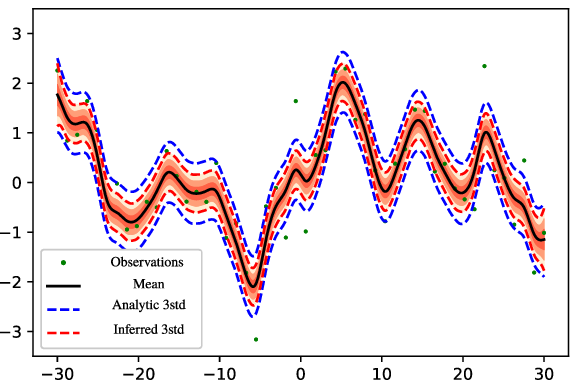}}
\hspace{15pt}
\subfigure[PBBVI with $K=3$]{\includegraphics[width=0.42 \textwidth]{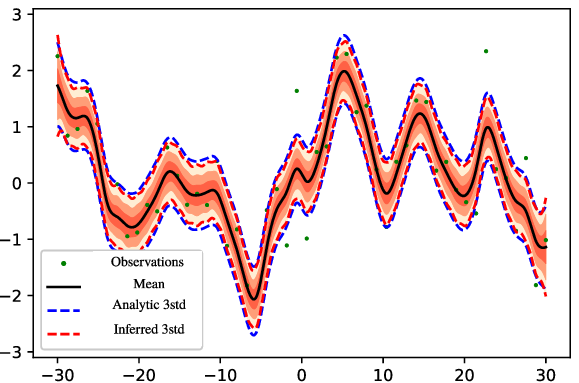}}
\caption{Gaussian process regression on synthetic data (green dots). Three standard deviations are shown in varying shades of orange.
The blue dashed lines show three standard deviations of the true posterior. The red dashed lines show the inferred three standard deviations using KLVI~(a) and PBBVI~(b). We see that the results from our proposed PBBVI are close to the analytic solution while traditional KLVI underestimates the variances.}
\label{fig:Toy_GP_fig}
\vspace{4pt}
\end{figure}

\begin{table}[t]
\subfigure[GP regression]{
    \begin{tabular}{l c }
     \toprule
     Method &Avg variances\\
     \cmidrule(lr){1-2}
     Analytic & 0.0415 \\
     KLVI & 0.0176 \\
     PBBVI &  0.0355 \\
     \bottomrule
    \end{tabular}}
    \hspace{5pt}
    \subfigure[GP classification]{
    \begin{tabular}{ l  c c c c}
     \toprule
    Data set &Crab &Pima & Heart & Sonar   \\
     \cmidrule(lr){1-5}
    KLVI &0.22& 0.245   & 0.148 & 0.212 \\
    PBBVI & \textbf{0.11} & \textbf{0.240}  & \textbf{0.133} & \textbf{0.173} \\
     \bottomrule
    \end{tabular}}
    \caption{Results for Gaussian process (GP) experiments.
    (a)~Average marginal posterior variances at the positions of data points for GP regression with synthetic data (see Figure~\ref{fig:Toy_GP_fig}).
    The proposed PBBVI is closer to the analytic solution than standard KLVI.
    (b)~Error rate of GP classification on the test set. The lower the better. Our proposed PBBVI consistently obtains better classification results. }
    \label{tab:gp}
\end{table}

In this section, we inspect the inference behavior using a synthetic data set with Gaussian processes (GP).  We generate the data according to a Gaussian noise distribution centered around a mixture of sinusoids, and we sample 50 data points (green dots in Figure~\ref{fig:Toy_GP_fig}).
We then use a GP to model the data, thus assuming the generative process $f \sim \mathcal{GP} (0, \Lambda)$ and $y_i \sim \mathcal{N}(f_i, \epsilon)$.

In this experiment, the model admits an analytic solution (three standard deviations shown in blue dashed lines in Figure~\ref{fig:Toy_GP_fig}).
We compare the analytic solution to approximate posteriors using a fully factorized Gaussian variational distribution obtained by KLVI (Figure~\ref{fig:Toy_GP_fig}~(a)) and by the proposed PBBVI (Figure~\ref{fig:Toy_GP_fig}~(b)).
The results from PBBVI are almost identical to the analytic solution.
In contrast, KLVI underestimates the posterior variance. This is consistent with Table \ref{tab:gp}~(a), which shows the average marginal variances.
PBBVI results are much closer to the analytic results.

\subsection{Gaussian Process Classification}
\label{sec:GP_classification}

We evaluate the performance of PBBVI and KLVI on a GP classification task. Since the model is non-conjugate, no analytical baseline is available in this case. We model the data with the following generative process:
\begin{equation} \label{eq:gp-class-model}
    f \sim \mathcal{GP}(0, \Lambda),
    \qquad z_i = \sigma(f_i),
    \qquad y_i \sim Bern(z_i).
\end{equation}
Above, $\Lambda$ is the GP kernel, $\sigma$ indicates the sigmoid function, and $Bern$ indicates the Bernoulli distribution. We furthermore use the Matern-32 kernel,
\begin{equation} \label{eq:gp-class-kernel}
\Lambda_{ij} = s^2 \Bigg(1 + \frac{\sqrt{3}\, r_{ij}}{l}\Bigg) \exp\left( - \textstyle \frac{\sqrt{3}\, r_{ij}}{l}\right),
    \qquad
r_{ij}={||x_i-x_j||}_2 \; .
\end{equation}

\paragraph{Data.}
We use four data sets from the UCI machine learning repository, suitable for binary classification: Crab (200 datapoints), Pima (768 datapoints), Heart (270 datapoints), and Sonar (208 datapoints).
We randomly split each of the data sets into two halves. One half is used for training and the other half is used for testing. We set the hyper parameters $s = 1$ and $l = \sqrt{D}/2$ throughout all experiments, where $D$ is the dimensionality of the input $x$.

Table~\ref{tab:gp}~(b) shows the classification performance (error rate) for these data sets. Our proposed PBBVI consistently performs better than the traditional KLVI.

\begin{figure}
    \centering
    \subfigure[Speed of convergence]{\includegraphics[height=1.6in]{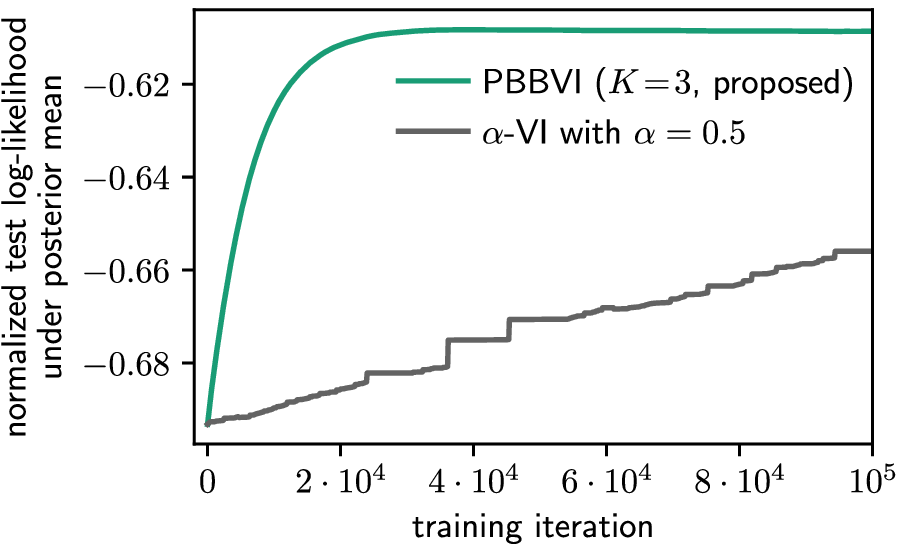}}
    \hspace{15pt}
    \subfigure[Gradient variance]{\includegraphics[height=1.6in]{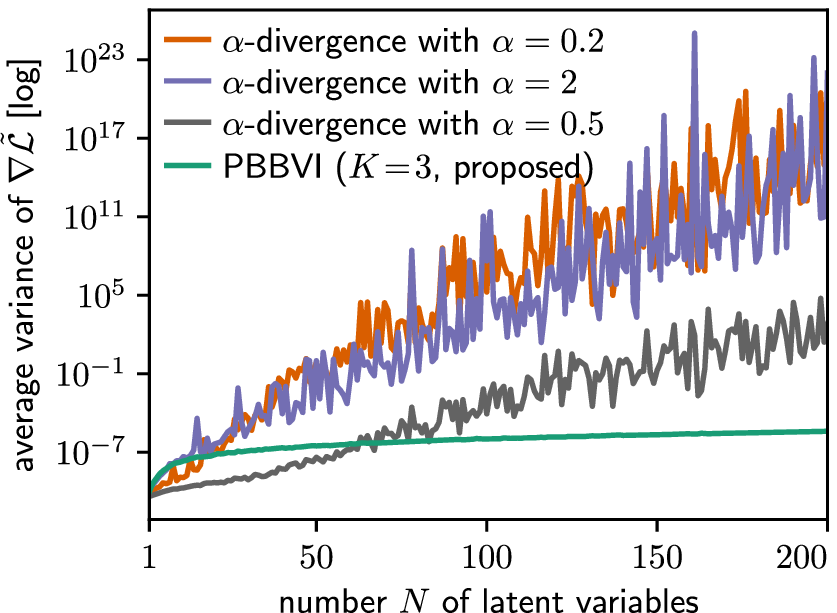}}
    \caption{Comparisons between PBBVI (proposed) and $\alpha$-VI.
    (a)~Training curves (test log-likelihood per data point) for GP classification on the Sonar data set. PBBVI converges faster than $\alpha$-VI even though we tuned the number of MC samples per training step ($100$) and the constant learning rate ($10^{-5}$) so as to maximize the performance of $\alpha$-VI on a validation set.
    (b)~Sampling variance of the stochastic gradient (averaged over its components) in the optimum of a GP regression model with synthetic data, for $\alpha$-divergences (orange, purple, gray), and the proposed PBBVI (green). The variance grows exponentially with the latent dimension $N$ for $\alpha$-VI, and only algebraically for PBBVI.}
    \label{fig:variances}
\end{figure}

\paragraph{Convergence speed comparison.}
We also carry out a comparison in terms of speed of convergence, focusing on PBBVI and $\alpha$-VI.
Our results indicate that the smaller variance of the reparameterization gradient in PBBVI leads to faster convergence of the optimization algorithm.

We train the GP classifier from Eqs.~\ref{eq:gp-class-model}-\ref{eq:gp-class-kernel} on the Sonar UCI data set using a constant learning rate.
Figure \ref{fig:variances}~(a) shows the test log-likelihood under the posterior mean as a function of training iterations.
We split the data set into equally sized training, validation, and test sets.
We then tune the learning rate and the number of Monte Carlo samples per gradient step to obtain optimal performance on the validation set after optimizing the $\alpha$-VI bound with a fixed budget of random samples.
We use $\alpha=0.5$ here; smaller values of $\alpha$ (i.e., stronger mass-covering effect) leads to even slower convergence.
We then optimize the PBBVI lower bound with the same learning rate and number of Monte Carlo samples.
The final test error rate is $22\%$ (the data set has binary labels and is approximately balanced).
Although the hyperparameters are tuned for $\alpha$-VI, the proposed PBBVI converges an order of magnitude faster (Figure \ref{fig:variances}~(a)).

Figure~\ref{fig:variances}~(b) provides more insight in the scaling of the gradient variance.
Here, we fit GP regression models on synthetically generated data by maximizing the PBBVI lower bound and the $\alpha$-VI lower bound with $\alpha\in\{0.2, 0.5, 2\}$.
We generate a separate synthetic data set for each $N\in\{1,\ldots,200\}$ by drawing $N$ random data points around a sinusoidal curve.
For each $N$, we fit a one-dimensional GP regression with PBBVI and $\alpha$-VI, respectively, using the same data set for both methods.
The variational distribution is a fully factorized Gaussian with $N$ latent variables.
After convergence, we estimate the sampling variance of the gradient of each lower bound with respect to the posterior mean.
We calculate the empirical variance of the gradient based on $10^5$ samples from $q_\lambda$, and we average over the $N$ coordinates.
Figure~\ref{fig:variances} shows the average sampling variance as a function of $N$ on a logarithmic scale.
The variance of the gradient of the $\alpha$-VI bound grows exponentially in the number of latent variables.
By contrast, we find only algebraic growth for PBBVI.

\subsection{Variational Autoencoder}
\label{sec:VAE}

\begin{figure}
    \centering
    \includegraphics[width=0.5\textwidth]{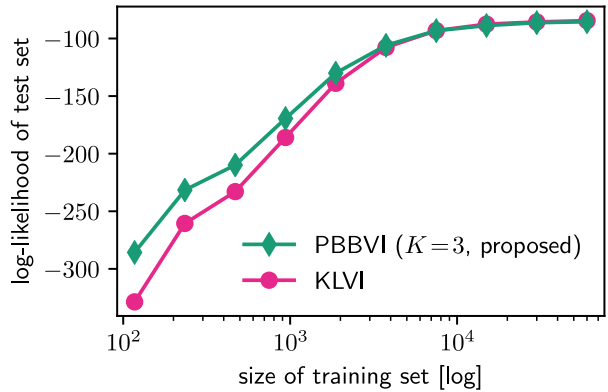}
    \caption{%
        Predictive likelihood of a VAE trained on data sets of different sizes.
        The higher value the better.
        The training data are randomly sampled subsets of the MNIST training set.
        Our proposed PBBVI method outperforms KLVI mainly when the size of the training data set is small.
        The fewer the training data, the more advantage PBBVI obtains.%
    }
    \label{fig:MNIST}
\end{figure}

We experiment on Variational Autoencoders (VAEs), and we compare the PBBVI and the KLVI bound in terms of predictive likelihoods on held-out data~\cite{kingma2013auto}.
Autoencoders compress unlabeled training data into low-dimensional representations by fitting it to an encoder-decoder model that maps the data to itself.
These models are prone to learning the identity function when the hyperparameters are not carefully tuned, or when the network is too expressive, especially for a moderately sized training set.
VAEs are designed to partially avoid this problem by estimating the uncertainty that is associated with each data point in the latent space.
It is therefore important that the inference method does not underestimate posterior variances.
We show that, for small data sets, training a VAE by maximizing the PBBVI lower bound leads to higher predictive likelihoods than maximizing the KLVI lower bound.

We train the VAE on the MNIST data set of handwritten digits~\cite{lecun1998gradient}.
We build on the publicly available implementation by~\cite{burda2015importance} and we also use the same architecture and hyperparamters, with $L=2$ stochastic layers and $S=5$ samples from the variational distribution per gradient step.
The model has $100$ latent units in the first stochastic layer and $50$ latent units in the second stochastic layer.

The VAE model factorizes over all data points.
We train it by stochastically maximizing the sum of the PBBVI lower bounds for all data points using a minibatch size of $20$.
The VAE amortizes the gradient signal across data points by training inference networks.
The inference networks express the mean and variance of the variational distribution as a function of the data point.
We add an additional inference network that learns the mapping from a data point to the reference energy $V_0$.
Here, we use a network with four fully connected hidden layers of $200$, $200$, $100$, and $50$ units, respectively.

MNIST contains $60{,}000$ training images.
To test our approach on smaller-scale data where Bayesian uncertainty matters more, we evaluate the test likelihood after training the model on randomly sampled fractions of the training set.
We use the same training schedules as in the publicly available implementation, keeping the total number of training iterations independent of the size of the training set.
Different to the original implementation, we shuffle the training set before each training epoch as this turns out to increase the performance for both our method and the baseline.

Figure \ref{fig:MNIST} shows the predictive log-likelihood of the whole test set, where the VAE is trained on random subsets of different sizes of the training set.
We use the same subset to train with PBBVI and KLVI for each training set size.
PBBVI leads to a higher predictive likelihood than traditional KLVI on subsets of the data. We explain this finding with our observation that the variational distributions obtained from PBBVI capture more of the posterior variance.
As the size of the training set grows---and the posterior uncertainty decreases---the performance of KLVI catches up with PBBVI.

As a potential explanation why PBBVI converges to the KLVI result for large training sets, we note that $\E_{q_{\lambda^*}}[(V_0^*-V)^3]=0$ at the optimal variational distribution $q_{\lambda^*}$ and reference energy $V_0^*$ (see Section~\ref{sec:proofs}).
If $V$ becomes a symmetric random variable (such as a Gaussian) in the limit of a large training set, then this implies that $\E_{q_{\lambda^*}}[V]=V_0^*$, and PBBVI with~$K=3$ reduces to KLVI for large training sets.

\section{Related work}
\label{sec:related}
Our approach is related to BBVI, VI with generalized divergences, and variational perturbation theory. We thus briefly discuss related work in these three directions.

\paragraph{Black box variational inference (BBVI).}
BBVI has already been addressed in the introduction~\cite{salimans2013fixed,kingma2013auto,rezende2014stochastic,ranganath2014black, ruiz2016generalized}; it enables variational inference for many models~\cite{ranganath2015deep,bamler2017dynamic,bamler2018improving,deng2017factorized,yingzhen2018disentangled,mescheder2017adversarial,ma2018eddi}. Recent developments include variance reduction and improvements in reparameterization and amortization~\cite{kingma2015variational,tucker2017rebar,rezende2015variational,mandt2014smoothed,ruiz2016generalized,buchholz2018quasi,marino2018iterative} which are all compatible with our approach.  Our work builds upon BBVI in that BBVI makes a large class of new divergence measures between the posterior and the approximating distribution tractable. Depending on the divergence measure, BBVI may suffer from high-variance stochastic gradients. This is a practical problem that we address in this paper.

\paragraph{Generalized divergences measures.}
Our work connects to generalized information-theoretic divergences~\cite{shun2012differential}. \cite{minka2005divergence} introduced a broad class of divergences for variational inference, including $\alpha$-divergences. Most of these divergences have been intractable in large-scale applications until the advent of BBVI.  In this context, $\alpha$-divergences were first suggested by ~\cite{hernandez2016black} for local divergence minimization, and later for global minimization by~\cite{li2016renyi} and~\cite{dieng2016chi}. As we show in this paper, $\alpha$-divergences have the disadvantage of inducing high-variance gradients, since the ratio between posterior and variational distribution enters the bound polynomially instead of logarithmically. In contrast, our approach leads to a more stable inference scheme in high dimensions.

\paragraph{Variational perturbation theory.}
Perturbation theory refers to methods that aim to truncate a typically divergent power series to a finite series.
In machine learning, these approaches have been addressed from an information-theoretic perspective by~\cite{tanaka1999theory,tanaka2000information}.
Thouless-Anderson-Palmer (TAP) equations \cite{thouless1977solution} are a form of second-order perturbation theory.
They were originally developed in statistical physics to include perturbative corrections to the mean-field solution of Ising models.
They have been adopted into Bayesian inference in~\cite{plefka1982convergence} and were advanced by many authors~\cite{kappen2001second,paquet2009perturbation,opper2013perturbative,opper2015expectation}.
In variational inference, perturbation theory yields extra terms to the mean-field variational objective which are difficult to calculate analytically.
This may be a reason why the methods discussed are not widely adopted by practitioners.
In this paper, we emphasize the ease of including perturbative corrections in a black box variational inference framework.
Furthermore, in contrast to earlier formulations, our approach yields a strict lower bound to the marginal likelihood which can be conveniently optimized.
Our approach is different from the traditional variational perturbation formulation~\cite{kleinert2009path,schwartz2008ideas}, which generally does not result in a bound.

\section{Conclusion}
\label{sec:conclusion}
We first presented a view on  black box variational inference as a form of biased importance sampling, where we can trade off bias versus variance by the choice of divergence.
Bias refers to the deviation of the bound from the true marginal likelihood, and variance refers to its reparameterization gradient estimator.
We then proposed a family of new variational bounds that connect to variational perturbation theory, and which include corrections to the standard Kullback-Leibler bound.
Our proposed PBBVI bound is an asymptotic series to the true marginal likelihood for large order $K$ of the perturbative expansion, and we showed both theoretically and experimentally that it has lower-variance reparameterization gradients compared to $\alpha$-VI.
In order to scale up our method to massive data sets, future work will explore stochastic versions of PBBVI.
Since the PBBVI bound contains interaction terms between all data points, breaking it up into mini-batches is not straightforward.
Besides, while our experiments used a fixed perturbative order of $K=3$, it could be beneficial to increase the perturbative order at some point during the training cycle once an empirical estimate of the gradient variance drops below a certain threshold.
It would also be interesting to investigate a~$K$-independent formulation of PBBVI using Russian roulette estimates~\cite{lyne2015russian}.
Furthermore, the PBBVI and $\alpha$-bounds can also be combined, such that PBBVI further approximates $\alpha$-VI.
This could lead to promising results on large data sets where traditional $\alpha$-VI is hard to optimize due to its variance, and traditional PBBVI converges to KLVI.
As a final remark, a tighter variational bound is not guaranteed to always result in a better posterior approximation since the variational family limits the quality of the solution.
However, in the context of variational EM, where one performs gradient-based hyperparameter optimization on the log marginal likelihood, our bound gives more reliable results since higher orders of $K$ can be assumed to approximate the marginal likelihood better.

\section*{References}
\bibliographystyle{plain}
\bibliography{ref}

\begin{thebibliography}{10}

\bibitem{shun2012differential}
Shunichi Amari.
\newblock {\em Differential-geometrical methods in statistics}, volume~28.
\newblock Springer Science \& Business Media, 2012.

\bibitem{bamler2017dynamic}
Robert Bamler and Stephan Mandt.
\newblock Dynamic word embeddings.
\newblock In {\em International Conference on Machine Learning}, pages
  380--389, 2017.

\bibitem{bamler2018improving}
Robert Bamler and Stephan Mandt.
\newblock Improving optimization in models with continuous symmetry breaking.
\newblock In {\em International Conference on Machine Learning}, pages
  432--441, 2018.

\bibitem{bamler2017perturbative}
Robert Bamler, Cheng Zhang, Manfred Opper, and Stephan Mandt.
\newblock Perturbative black box variational inference.
\newblock In {\em Advances in Neural Information Processing Systems}, pages
  5079--5088, 2017.

\bibitem{blei2017variational}
David~M Blei, Alp Kucukelbir, and Jon~D McAuliffe.
\newblock Variational inference: A review for statisticians.
\newblock {\em Journal of the American Statistical Association},
  112(518):859--877, 2017.

\bibitem{buchholz2018quasi}
Alexander Buchholz, Florian Wenzel, and Stephan Mandt.
\newblock Quasi-monte carlo variational inference.
\newblock In {\em International Conference on Machine Learning}, pages
  667--676, 2018.

\bibitem{burda2015importance}
Yuri Burda, Roger Grosse, and Ruslan Salakhutdinov.
\newblock Importance weighted autoencoders.
\newblock In {\em International Conference on Learning Representations}, pages
  1--9, 2016.

\bibitem{dempster1977maximum}
Arthur~P Dempster, Nan~M Laird, and Donald~B Rubin.
\newblock Maximum likelihood from incomplete data via the {EM} algorithm.
\newblock {\em Journal of the Royal Statistical Society. Series B
  (methodological)}, pages 1--38, 1977.

\bibitem{deng2017factorized}
Zhiwei Deng, Rajitha Navarathna, Peter Carr, Stephan Mandt, Yisong Yue, Iain
  Matthews, and Greg Mori.
\newblock Factorized variational autoencoders for modeling audience reactions
  to movies.
\newblock In {\em Proceedings of the IEEE Conference on Computer Vision and
  Pattern Recognition}, pages 2577--2586, 2017.

\bibitem{dieng2016chi}
Adji~B Dieng, Dustin Tran, Rajesh Ranganath, John Paisley, and David~M Blei.
\newblock Variational inference via $\chi$ upper bound minimization.
\newblock In {\em Advances in Neural Information Processing Systems}, pages
  2732--2741, 2017.

\bibitem{domke2018importance}
Justin Domke and Daniel~R Sheldon.
\newblock Importance weighting and variational inference.
\newblock In {\em Advances in Neural Information Processing Systems}, pages
  4470--4479, 2018.

\bibitem{hernandez2016black}
Jose Hernandez-Lobato, Yingzhen Li, Mark Rowland, Thang Bui, Daniel
  Hern{\'a}ndez-Lobato, and Richard Turner.
\newblock Black-box alpha divergence minimization.
\newblock In {\em International Conference on Machine Learning}, pages
  2256--2273, 2016.

\bibitem{hoffman2013stochastic}
Matthew~D Hoffman, David~M Blei, Chong Wang, and John~William Paisley.
\newblock Stochastic variational inference.
\newblock {\em Journal of Machine Learning Research}, 14(1):1303--1347, 2013.

\bibitem{jordan1999introduction}
Michael~I Jordan, Zoubin Ghahramani, Tommi~S Jaakkola, and Lawrence~K Saul.
\newblock An introduction to variational methods for graphical models.
\newblock {\em Machine learning}, 37(2):183--233, 1999.

\bibitem{kappen2001second}
Hilbert~J Kappen and Wim Wiegerinck.
\newblock Second order approximations for probability models.
\newblock pages 238--244, 2001.

\bibitem{kingma2013auto}
Diederik~P Kingma and Max Welling.
\newblock Auto-encoding variational {Bayes}.
\newblock In {\em International Conference on Learning Representations}, pages
  1--9, 2014.

\bibitem{kingma2015variational}
Durk~P Kingma, Tim Salimans, and Max Welling.
\newblock Variational dropout and the local reparameterization trick.
\newblock In {\em Advances in Neural Information Processing Systems}, pages
  2575--2583, 2015.

\bibitem{kleinert2009path}
Hagen Kleinert.
\newblock {\em Path integrals in quantum mechanics, statistics, polymer
  physics, and financial markets}.
\newblock World scientific, 2009.

\bibitem{lecun1998gradient}
Yann LeCun, L{\'e}on Bottou, Yoshua Bengio, and Patrick Haffner.
\newblock Gradient-based learning applied to document recognition.
\newblock volume~86, pages 2278--2324, 1998.

\bibitem{yingzhen2018disentangled}
Yingzhen Li and Stephan Mandt.
\newblock Disentangled sequential autoencoder.
\newblock In {\em International Conference on Machine Learning}, pages
  5656--5665, 2018.

\bibitem{li2016renyi}
Yingzhen Li and Richard~E Turner.
\newblock R{\'e}nyi divergence variational inference.
\newblock In {\em Advances in Neural Information Processing Systems}, pages
  1073--1081, 2016.

\bibitem{lyne2015russian}
Anne-Marie Lyne, Mark Girolami, Yves Atchad{\'e}, Heiko Strathmann, Daniel
  Simpson, et~al.
\newblock On russian roulette estimates for bayesian inference with
  doubly-intractable likelihoods.
\newblock {\em Statistical science}, 30(4):443--467, 2015.

\bibitem{ma2018eddi}
Chao Ma, Sebastian Tschiatschek, Konstantina Palla, Jose Miguel~Hernandez
  Lobato, Sebastian Nowozin, and Cheng Zhang.
\newblock {EDDI:} efficient dynamic discovery of high-value information with
  partial {VAE}.
\newblock In {\em International Conference on Machine Learning}, pages 1--8,
  2019.

\bibitem{mandt2014smoothed}
Stephan Mandt and David Blei.
\newblock Smoothed gradients for stochastic variational inference.
\newblock In {\em Advances in Neural Information Processing Systems}, pages
  2438--2446, 2014.

\bibitem{marino2018iterative}
Joseph Marino, Yisong Yue, and Stephan Mandt.
\newblock Iterative amortized inference.
\newblock In {\em International Conference on Machine Learning}, pages
  3400--3409, 2018.

\bibitem{mescheder2017adversarial}
Lars Mescheder, Sebastian Nowozin, and Andreas Geiger.
\newblock Adversarial variational bayes: Unifying variational autoencoders and
  generative adversarial networks.
\newblock In {\em International Conference on Machine Learning}, pages
  2391--2400, 2017.

\bibitem{minka2005divergence}
Thomas Minka.
\newblock Divergence measures and message passing.
\newblock Technical report, Technical report, Microsoft Research, 2005.

\bibitem{opper2015expectation}
Manfred Opper.
\newblock Expectation propagation.
\newblock In {\em Statistical {Physics}, {Optimization}, {Inference}, and
  {Message}-{Passing} {Algorithms}}, chapter~9, pages 263--292. 2015.

\bibitem{opper2015perturbation}
Manfred Opper, Marco Fraccaro, Ulrich Paquet, Alex Susemihl, and Ole Winther.
\newblock Perturbation theory for variational inference.
\newblock In {\em Advances in Neural Information Processing Systems Workshop on
  Advances in Approximate Bayesian Inference}, 2015.

\bibitem{opper2013perturbative}
Manfred Opper, Ulrich Paquet, and Ole Winther.
\newblock Perturbative corrections for approximate inference in gaussian latent
  variable models.
\newblock {\em Journal of Machine Learning Research}, 14(1):2857--2898, 2013.

\bibitem{paquet2009perturbation}
Ulrich Paquet, Ole Winther, and Manfred Opper.
\newblock Perturbation corrections in approximate inference: Mixture modelling
  applications.
\newblock {\em Journal of Machine Learning Research}, 10(Jun):1263--1304, 2009.

\bibitem{plefka1982convergence}
Timm Plefka.
\newblock Convergence condition of the {TAP} equation for the infinite-ranged
  ising spin glass model.
\newblock {\em Journal of Physics A: Mathematical and general}, 15(6):1971,
  1982.

\bibitem{rainforth2018tighter}
Tom Rainforth, Adam Kosiorek, Tuan~Anh Le, Chris Maddison, Maximilian Igl,
  Frank Wood, and Yee~Whye Teh.
\newblock Tighter variational bounds are not necessarily better.
\newblock In {\em International Conference on Machine Learning}, pages
  4277--4285, 2018.

\bibitem{ranganath2014black}
Rajesh Ranganath, Sean Gerrish, and David~M Blei.
\newblock Black box variational inference.
\newblock In {\em International Conference on Artificial Intelligence and
  Statistics}, pages 814--822, 2014.

\bibitem{ranganath2015deep}
Rajesh Ranganath, Linpeng Tang, Laurent Charlin, and David Blei.
\newblock Deep exponential families.
\newblock In {\em Artificial Intelligence and Statistics}, pages 762--771,
  2015.

\bibitem{rezende2015variational}
Danilo~Jimenez Rezende and Shakir Mohamed.
\newblock Variational inference with normalizing flows.
\newblock In {\em International Conference on International Conference on
  Machine Learning}, pages 1530--1538, 2015.

\bibitem{rezende2014stochastic}
Danilo~Jimenez Rezende, Shakir Mohamed, and Daan Wierstra.
\newblock Stochastic backpropagation and approximate inference in deep
  generative models.
\newblock In {\em International Conference on Machine Learning}, pages
  1278--1286, 2014.

\bibitem{robbins1951stochastic}
Herbert Robbins and Sutton Monro.
\newblock A stochastic approximation method.
\newblock {\em The annals of mathematical statistics}, 1951.

\bibitem{ruiz2016generalized}
Francisco Ruiz, Michaelis Titsias, and David Blei.
\newblock The generalized reparameterization gradient.
\newblock In {\em Advances in neural information processing systems}, pages
  460--468, 2016.

\bibitem{salimans2013fixed}
Tim Salimans and David~A Knowles.
\newblock Fixed-form variational posterior approximation through stochastic
  linear regression.
\newblock {\em Bayesian Analysis}, 8(4):837--882, 2013.

\bibitem{schwartz2008ideas}
Moshe Schwartz and Eytan Katzav.
\newblock The ideas behind self-consistent expansion.
\newblock {\em Journal of Statistical Mechanics: Theory and Experiment},
  2008(04):23, 2008.

\bibitem{tanaka1999theory}
Toshiyuki Tanaka.
\newblock A theory of mean field approximation.
\newblock In {\em Advances in Neural Information Processing Systems}, pages
  1--8, 1999.

\bibitem{tanaka2000information}
Toshiyuki Tanaka.
\newblock Information geometry of mean-field approximation.
\newblock {\em Neural Computation}, 12(8):1951--1968, 2000.

\bibitem{thouless1977solution}
David~J Thouless, Philip~W Anderson, and Robert~G Palmer.
\newblock Solution of `solvable model of a spin glass'.
\newblock {\em Philosophical Magazine}, 35(3):593--601, 1977.

\bibitem{tucker2017rebar}
George Tucker, Andriy Mnih, Chris~J Maddison, John Lawson, and Jascha
  Sohl-Dickstein.
\newblock Rebar: Low-variance, unbiased gradient estimates for discrete latent
  variable models.
\newblock In {\em Advances in Neural Information Processing Systems}, pages
  2627--2636, 2017.

\bibitem{zhang2018advances}
Cheng Zhang, Judith Butepage, Hedvig Kjellstrom, and Stephan Mandt.
\newblock Advances in variational inference.
\newblock {\em IEEE transactions on pattern analysis and machine intelligence},
  pages 1--20, 2018.

\end{thebibliography}

\end{document}